\documentclass[11pt]{article}

\usepackage{acl}

\usepackage{times}
\usepackage{latexsym}
\usepackage{etoc}
\usepackage{booktabs}
\usepackage[table]{xcolor}
\usepackage{multirow}
\usepackage{array}
\usepackage{makecell}
\usepackage{caption}
\usepackage[most]{tcolorbox}
\usepackage{xcolor}

\usepackage{amsmath,amssymb,amsthm}

\definecolor{lightpurple}{rgb}{0.862, 0.918, 0.992}

\newtheorem{theorem}{Theorem}[section]

\newtheorem{definition}[theorem]{Definition}
\newtheorem{assumption}[theorem]{Assumption}
\usepackage[ruled,vlined]{algorithm2e}

\usepackage[T1]{fontenc}

\usepackage[utf8]{inputenc}

\usepackage{microtype}

\usepackage{inconsolata}

\usepackage{graphicx}

\title{Rethinking Data Mixing from the Perspective of Large Language Models}

\author{
  Yuanjian Xu\textsuperscript{1}, 
  Tianze Sun\textsuperscript{3}, 
  Changwei Xu\textsuperscript{2}, 
  XinLong Zhao\thanks{China Mining Group}, 
  Jianing Hao\textsuperscript{1}, \\
  \textbf{Ran Chen}\textsuperscript{2}, 
  \textbf{Yang Liu}\protect\footnotemark[1], 
  \textbf{Ruijie Xu}\textsuperscript{2}, 
  \textbf{Stephen Chen}\textsuperscript{2}, 
  \textbf{Guang Zhang}\textsuperscript{1,}\thanks{Corresponding author.} \\
  \textsuperscript{1}Hong Kong University of Science and Technology (Guangzhou) \\
  \textsuperscript{2}OpenCSG \quad \textsuperscript{3}Harbin Institute of Technology University\\
  \texttt{\{yxu085@connect, guangzhang@\}hkust-gz.edu.cn}
}
\begin{document}

\maketitle

\begin{abstract}
Data mixing strategy is essential for large language model (LLM) training. Empirical evidence shows that inappropriate strategies can significantly reduce generalization. Although recent methods have improved empirical performance, several fundamental questions remain open: what constitutes a domain, whether human and model perceptions of domains are aligned, and how domain weighting influences generalization. We address these questions by establishing formal connections between gradient dynamics and domain distributions, offering a theoretical framework that clarifies the role of domains in training dynamics. Building on this analysis, we introduce DoGraph, a reweighting framework that formulates data scheduling as a graph-constrained optimization problem.
Extensive experiments on GPT-2 models of varying scales demonstrate that DoGraph consistently achieves competitive  performance. Code and data are publicly available at \url{https://anonymous.4open.science/r/Dograph-53B9}.

\end{abstract}

\section{Introduction}

Training data fundamentally determines the capability of large language models (LLMs)~\citep{xu2023hard,wettig2024qurating,albalaksurvey}.
However, domain distributions are imbalanced due to unequal data availability: web-scale corpora are abundant, whereas specialized domains remain scarce~\citep{gao2021pile}.
This raises a key question—can we design a principled sampling strategy to mitigate such imbalance?
Exhaustively searching over all possible sampling policies is infeasible, as LLM training is prohibitively expensive. To make progress, we must first answer: what does a “domain” truly mean for a LLM, and are human and model perceptions of domains aligned~\citep{sun2025domain2vec}?

Prior data mixing studies have predominantly relied on domain definitions derived from human intuition. Existing approaches can be broadly categorized into two lines of work. The first derives heuristics from small- or medium-scale models and then scales them to LLMs~\citep{liu2024regmix,ye2024mixinglaws,fan2023doge,xie2023doremi}; however, empirical evidence shows that scaling laws and domain sensitivities observed in small models do not transfer reliably to larger ones~\citep{kang2024autoscale}. The second directly performs data reweighting or optimization on LLMs, either at the sample or domain level~\citep{sun2025domain2vec,sow2025dynamic}, but often incurs prohibitive computational costs or relies on unrealistic assumptions.

In this work, we argue that the optimization of LLMs continuously reshapes their domain perception, creating a mismatch between human-defined and model-internal representations~\citep{bengio2013representation}.
Figure~\ref{fig:domain_evolution} visualizes this evolution: at initialization, samples from domains such as C4, Wikipedia, Book, and ArXiv form well-separated clusters, reflecting strong domain-specific biases.
As training progresses, these clusters gradually merge into an approximately isotropic distribution, indicating that the model internalizes more domain-invariant linguistic structures~\citep{power2022grokking,gao2019representation}.

\begin{figure*}[t]
\centering
\includegraphics[width=\textwidth]{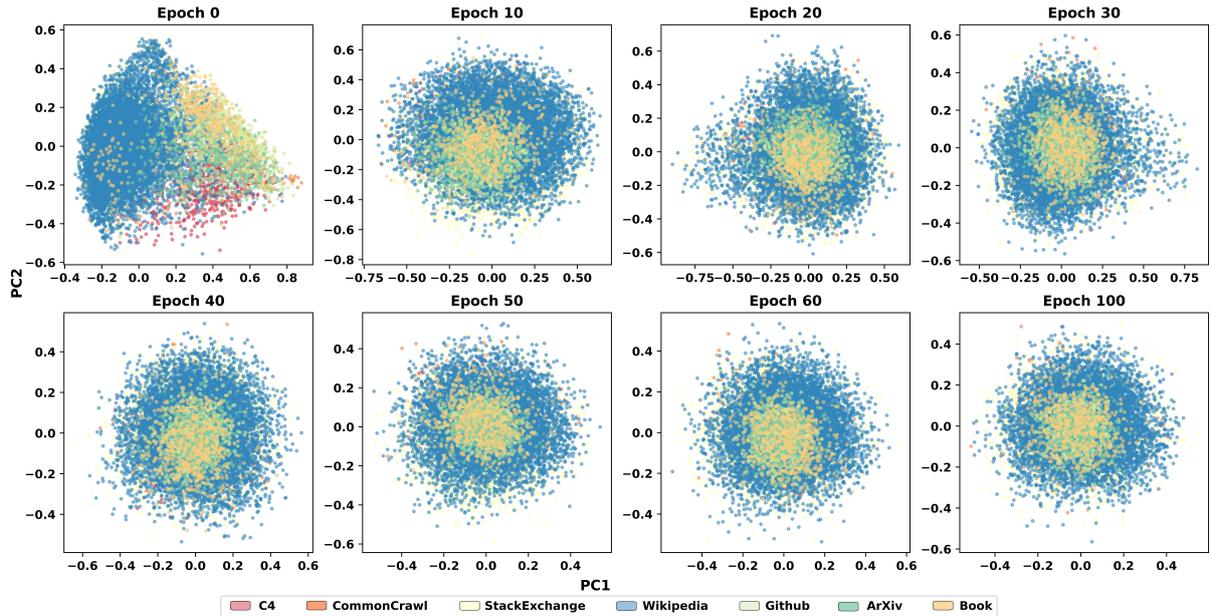}
\caption{
PCA projections of gradient directions at different training epochs. Colors denote data domains (C4, Wikipedia, ArXiv, Book, etc.). Initially, gradients form distinct clusters, showing strong domain bias. Over time, they overlap, indicating that the model homogenizes its domain perception. Experiments use 20\% of SlimPajama trained on GPT2-Mini.
}
\label{fig:domain_evolution}
\end{figure*}

This evolving misalignment biases existing data mixing methods.
To address it, we formally link domain distributions with gradient dynamics, showing how model-defined domains emerge during optimization~\citep{koh2017understanding,fort2019deep}.
Building on this foundation, we propose \emph{DoGraph}, which formulates domain scheduling as a graph-constrained reweighting problem.
DoGraph models the model-perceived domains as graph nodes and learns their weights through optimization.
Our main contributions are summarized as follows: 1) We theoretically establish a connection between domain distribution and gradient dynamics, and empirically validate the dynamic correction of domain representations during LLM training. 2) We propose DoGraph, a graph-constrained reweighting framework that formalizes domain scheduling as an optimization problem. DoGraph is strongly grounded in theoretical principles. 3) We conduct extensive experiments across diverse benchmarks, demonstrating consistent improvements in both performance and domain balance, which validate the competitiveness of our approach.

\section{Methods}
In this section, we begin by redefining domains from a learning-theoretic perspective.  
We then establish their connection to gradients, showing that distributional differences are reflected in gradient geometry in Section~\ref{sec:rethinking_the_definition_of_domain}. Finally, we build on this insight to propose the \emph{DoGraph}.
\subsection{Rethinking the Definition of Domain}  
\label{sec:rethinking_the_definition_of_domain}
In NLP, the notion of domain has often been unclearly defined, particularly in the training corpora of LLMs, where such boundaries are increasingly blurred.  
Before developing concrete strategies for domain weighting, it is essential to first clarify what we mean by a domain.  We argue that a domain should be defined from the model’s perspective, namely as the distribution of inputs it perceives, rather than from a human perspective.  
\begin{tcolorbox}[colback=lightpurple, colframe=lightpurple, sharp corners=all, boxrule=0mm, boxsep=0.5mm, left=1.5mm, right=1.5mm, top=1.5mm, bottom=1.5mm] 
\begin{definition}
\label{sec:definition}
Let $\mathcal V$ be a finite vocabulary and $\mathcal X = \mathcal V^*$ the space of all token sequences. A \emph{domain} is a probability space $(\mathcal X,\mathcal F,P_X)$, where $\mathcal F$ is the $\sigma$-algebra on $\mathcal X$ and $P_X$ a probability measure. Two domains $\mathcal D_1=(\mathcal X,\mathcal F,P_1)$ and $\mathcal D_2=(\mathcal X,\mathcal F,P_2)$ are distinct iff $P_1 \neq P_2$. 
\end{definition} 
\end{tcolorbox}
We now formulate the definition of domain as stated in Definition~\ref{sec:definition}.  
Each $x \in \mathcal X$ is a finite token sequence from $\mathcal V$, with domains distinguished by the regions of $\mathcal X$ where their distributions $P_X$ concentrate. In simple cases, such as distinguishing code from natural language, these regions are relatively easy to separate.  
However, in practice, many domains are much less clear-cut, with boundaries that overlap or gradually shift.  
Thus, domains differ through probability measures over the same space $\mathcal V^*$ rather than through disjoint supports.   

\paragraph{Connection between Domain and Gradients}  
A central question is whether domains can be inferred from observable data instead of being imposed a priori,  
as such assumptions inevitably risk introducing bias.  
Since each training sample affects learning only via its gradient, the model perceives not raw token frequencies but the geometry of gradient flows.  
To investigate this, we analyze a simplified self-attention structure in which the Transformer can be linearized,  
leading to a tractable correspondence between distributions and gradients.  

\begin{tcolorbox}[colback=lightpurple, colframe=lightpurple, sharp corners=all, boxrule=0mm, boxsep=0.5mm, left=1.5mm, right=1.5mm, top=1.5mm, bottom=1.5mm] 
\begin{theorem}
\label{thm:dist-grad}
Under the linearized Transformer setting,  
for any parameter block $b\in\{V,Q,K,O,W\}$ and two domains $P_1, P_2$,  
the difference of expected gradients satisfies
\(
\bar g_b(P_1)-\bar g_b(P_2) 
= \int \nabla_{W_b} L(x,y;\theta)\,
   (P_1-P_2)(dx,dy).
\)
Moreover, this difference admits a kernel representation:
\[
\|\bar g_b(P_1)-\bar g_b(P_2)\|^2 
= \mathrm{MMD}_{k_b}^2(P_1,P_2),
\]
where $k_b(s,s')=\langle g_b(s),g_b(s')\rangle$ is the gradient-induced kernel.  
\end{theorem}
\end{tcolorbox}

Theorem~\ref{thm:dist-grad} shows that distributional differences are encoded in the geometry of gradients,  
implying that domains can be compared through their gradient signatures rather than token-level statistics.  
From this perspective, a domain is defined by its expected gradient flow,  
and training can be understood as a continual refinement of the model’s perception of domains,  
with each update adjusting how distributions are represented in gradient space.  

\subsection{DoGraph}
We argue that data weighting should adapt to the model’s evolving perception of domains, 
rather than fixed human-defined boundaries. 
Building on this idea, we introduce the \emph{DoGraph}, 
where each domain corresponds to a node in a graph. 
At every epoch, we collect per-sample gradients and \textbf{project them into a low-dimensional subspace via random projection}.
Next, we apply K-means clustering in the projected gradient space to obtain model-centric partitions of the training distribution.
This partition evolves over training, reflecting the changing geometry of gradients.

Formally, let \( g_i \in \mathbb{R}^d \) be the gradient of the \(i\)-th sample and 
\( G = [g_1, \dots, g_n]^\top \in \mathbb{R}^{n \times d} \). 
We apply a random projection matrix \( R \in \mathbb{R}^{d \times k} \) 
with \( R_{pq} \sim \mathcal{N}(0, 1/k) \), yielding
\(
\tilde g_i = R^\top g_i, \quad \text{or } \tilde G = G R.
\)
By the Johnson--Lindenstrauss lemma,
\[
(1 - \epsilon)\|g_i - g_j\|_2^2 
\le \|\tilde g_i - \tilde g_j\|_2^2 
\le (1 + \epsilon)\|g_i - g_j\|_2^2,
\]
ensuring that clustering in the projected space preserves the gradient geometry 
while reducing computational cost and noise. Clustering \(\{\tilde g_i\}\) into \(m\) groups \( \{D_1, \dots, D_m\} \), 
we compute each domain’s mean gradient as
\(
\bar g_j = \frac{1}{|D_j|}\sum_{i \in D_j} \tilde g_i.
\)
To balance learning, we assign adaptive domain weights \(w = (w_1, \dots, w_m)\) by solving
\(
\min_{w \in \Delta^{m-1}} 
\; \mathcal{L}_{\text{opt}}\!\left(\sum_{j=1}^m w_j \bar g_j\right),
\)
where \( \Delta^{m-1} \) is the probability simplex.

\paragraph{DoGraph Pipeline}  
At each epoch, per-sample gradients are first extracted and projected into a low-dimensional subspace via random projection,  
then clustered into domains in the projected space.  
Domain mean gradients are aggregated through an optimization step that computes the optimal domain weights.  
The model parameters are updated with the weighted gradient, and the process repeats,  
allowing both the partition of domains and their relative importance to adapt continuously throughout training.  
The choice of the optimization objective $\mathcal{L}_{\text{opt}}$   
is discussed in the Appendix~\ref{sec:more_analysis_about}. Algorithm~\ref{alg:dograph} summarizes the overall procedure of DoGraph. 
More implementation details can be found in Appendix~\ref{sec:dograph_pipeline}.

\section{Experiment Results}
In this section, we begin by outlining the experimental setup,  
after which we present the overall performance analysis.  
We further conduct a perplexity analysis and investigate how model scale influences 
the observed trends, with detailed results presented in Section~\ref{sec:perplexity_analysis} 
and Section~\ref{sec:model_scale_influence}. All main results in the paper are reported using the GPT-2 Medium. To isolate the effects of architecture and parameter scale, additional experiments with the LLaMA-1.1B model are deferred to the Appendix~\ref{sec:Scaling to Larger Datasets and Model Sizes}. Sensitivity to hyperparameters and the choice of optimizer are analyzed in Appendix~\ref{sec:more_analysis_about} and Appendix~\ref{sec:impact}.

\begin{table*}[!t]
\centering
\small
\resizebox{2\columnwidth}{!}{
\setlength{\tabcolsep}{4pt}
\begin{tabular}{l cccccc | cc | c | c}
\toprule
& \multicolumn{6}{c|}{\textbf{Commonsense / Reasoning}} & \multicolumn{2}{c|}{\textbf{RC}} & \multicolumn{1}{c|}{\textbf{LM}} & \textbf{Avg} \\
\cmidrule(lr){2-7} \cmidrule(lr){8-9} \cmidrule(lr){10-10} \cmidrule(lr){11-11}
\textbf{Method} & HellaSwag & PiQA & OBQA & COPA & LogiQA & WinoG & SciQ & ARC-E & Lambada &  \\
\midrule
\multicolumn{11}{c}{SlimPajama} \\
\midrule
Uniform              & 26.1 & 55.5 & 11.7 & 58.0 & 25.7 & 49.9 & 49.0 & 31.4 & 11.6 & 35.4 \\
Dynamic Loss-Based   & 26.6 & 56.8 & 13.8 & 59.0 & \textbf{29.8} & 50.1 & 53.3 & 31.7 & 13.4 & 37.2 \\
DoReMi               & 26.4 & 55.7 & 12.2 & 59.0 & 27.2 & 49.9 & 53.3 & 32.3 & 12.7 & 36.5 \\
DOGE                 & 26.2 & 55.8 & 11.5 & 62.0 & 27.2 & 50.4 & 52.8 & 31.3 & 11.6 & 36.5 \\
RegMix               & 26.1 & 55.6 & 13.2 & 60.0 & 23.7 & 50.0 & 46.6 & 31.7 & 14.0 & 35.7 \\
Data Mixing Law      & 26.5 & 54.5 & 13.0 & 62.0 & 24.4 & 49.1 & 45.2 & 32.0 & 12.0 & 35.4 \\
\midrule
\rowcolor{gray!20}
DoGraph (Ours) 
& \textbf{27.3} & \textbf{56.9} & \textbf{14.8} & \textbf{63.0} & 26.3 & \textbf{50.8} 
& \textbf{53.5} & \textbf{33.9} & \textbf{14.5} & \textbf{37.9} \\
\bottomrule
\end{tabular}
}
\caption{
Downstream benchmark results (\textbf{accuracy \%}) on SlimPajama (GPT-2 Medium).
Tasks grouped into \textbf{Commonsense/Reasoning}, \textbf{Reading Comprehension}, and \textbf{Language Modeling}.
Best results highlighted in bold.
}
\label{tab:benchmark_joint}
\end{table*}
\subsection{Experiments Setup}
\label{sec:exp_setup}
Our experiments use decoder-only, Transformer-based language models \cite{vaswani2017attention, radford2019language} at 210M and 300M scales.  Models are trained on SlimPajama \cite{soboleva2023slimpajama}, spanning seven text domains.  
We evaluate DoGraph on nine stable benchmarks and compare with representative baselines.  
Model details, training protocol, and baseline breakdown are in Appendix~\ref{app:exp_details}.

\subsection{Results in the Pretraining Stage}
\label{sec:results_in_the_pretrainng}

As shown in Table~\ref{tab:benchmark_joint}, DoGraph achieves more balanced learning across domains and delivers consistent gains over all baselines.  
It yields the largest improvements on reasoning-oriented benchmarks, highlighting the advantage of its structured weighting mechanism in capturing logical and commonsense dependencies across domains.  
Moreover, the performance gains on reading comprehension tasks, which require semantic consistency and information integration, demonstrate that DoGraph’s adaptive data scheduling enhances semantic alignment and improves overall generalization.

\subsection{Perplexity Analysis}
\label{sec:perplexity_analysis}
Table~\ref{tab:pretrain_results} shows validation perplexity on SlimPajama under various domain-mixing strategies.
Uniform sampling performs moderately but fails to balance domain frequencies.
Loss-based weighting and prior methods (DoReMi, DOGE) yield unstable gains, overfitting to high-resource domains and degrading on long-tail data.
RegMix and Data Mixing Law worsen this trend, with higher perplexity despite larger models.
\textbf{DoGraph} achieves the best perplexity , reflecting balanced domain integration and strong generalization.

\begin{table}[htbp]
\centering
\small
\resizebox{\columnwidth}{!}{
\setlength{\tabcolsep}{6pt}
\begin{tabular}{lc}
\toprule
\textbf{Method} & \textbf{SlimPajama (Val PPL ↓)} \\
\midrule
Uniform              & 4.13 \\
DYNAMIC LOSS-BASED   & 3.10 \\
DoReMi               & 3.30 \\
DOGE                 & 3.31 \\
RegMix               & 4.51 \\
Data Mixing Law      & 4.50 \\
\midrule
\rowcolor{gray!20}
\textbf{DoGraph (Ours)} & \textbf{3.09} \\
\bottomrule
\end{tabular}
}
\caption{Pre-training results on SlimPajama.
Validation perplexity (PPL) comparison across domain-mixing strategies. 
Lower values indicate better generalization. 
}
\label{tab:pretrain_results}
\end{table}

\subsection{DoGraph Stability across Model Scales}
\label{sec:model_scale_influence}
As shown in Figure~\ref{fig:domain_bias}, validation perplexity decreases with model scale, but the rate of improvement depends on the reweighting strategy.  
Uniform weighting yields consistently high perplexity, while RegMix offers partial gains that diminish as models grow.  
DoGraph achieves the lowest perplexity across all scales, validating its ability to dynamically balance domains.
\vspace{-8pt}
\begin{figure}[ht]
    \centering
    \includegraphics[width=0.8\linewidth]{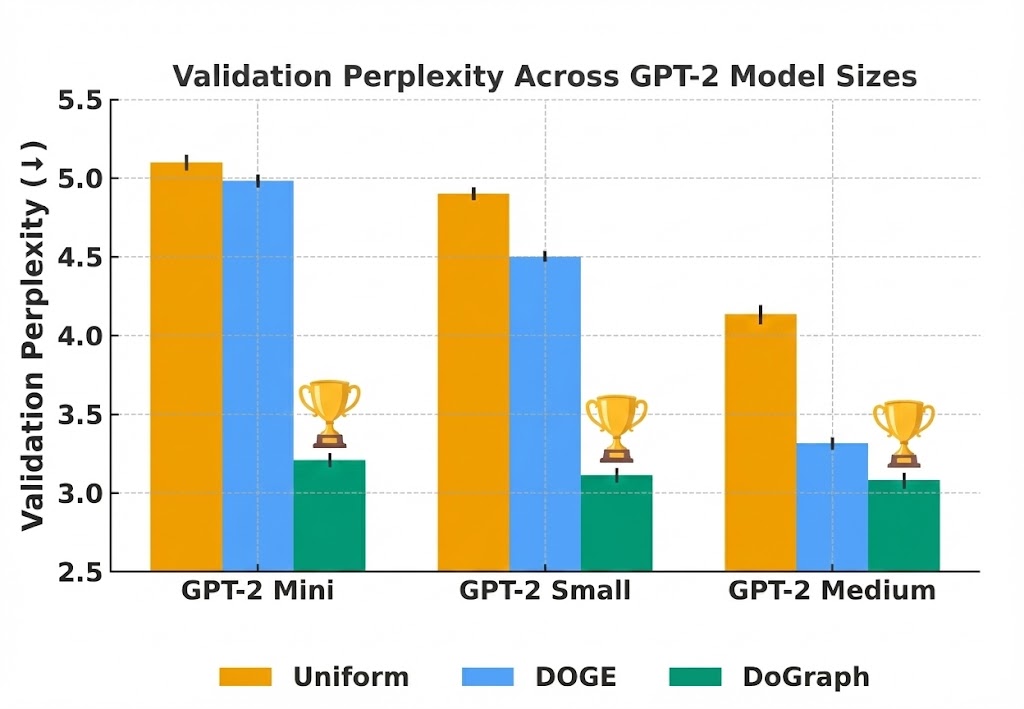}
    \caption{Perplexity across GPT-2 model sizes. }
    \label{fig:domain_bias}
\end{figure}

\section{Conclusion}
We revisited data mixing for LLMs through the lens of gradient dynamics.  
By characterizing domain differences via gradient geometry, we proposed \textbf{DoGraph}, 
a graph-constrained reweighting framework that adaptively balances domains during training.  
Experiments across model scales and benchmarks show that DoGraph improves both domain balance and generalization.  
Our results suggest that domains should be defined by the model’s evolving representation 
rather than human intuition.

\section{Limitations}
While DoGraph achieves consistent improvements across domains and already reduces computational overhead through randomized gradient projection, its efficiency can still be further optimized. Future work will explore more lightweight aggregation and weighting strategies to enhance scalability in large-scale training.

\bibliography{acl_2025}
\clearpage
\appendix
\section{Appendix}
\addcontentsline{toc}{section}{Appendix}
\etocsettocstyle{\section*{Contents}}{} 
\etocsetnexttocdepth{subsection} 
\localtableofcontents 
\subsection{Connections to Prior Work}
We categorize data mixture optimization into two main paradigms: offline and online approaches.

Offline approaches predefine mixture ratios before training. Early scaling-law studies~\citep{kaplan2020scaling,hoffmann2022chinchilla} established the relationship between model size, data volume, and compute, motivating subsequent work that explicitly models how mixture composition affects performance. Methods such as DoReMi~\citep{xie2023doremi}, RegMix~\citep{liu2024regmix}, and Mixing Laws~\citep{ye2024mixinglaws} optimize mixture ratios using proxy models or learned predictors, improving efficiency but requiring retraining when datasets change. Other efforts focus on heuristic sample scoring to derive refined data mixtures~\citep{gu2024data}, distinct from large-scale corpora that offer fixed domain ratios for benchmarking~\citep{gao2021pile,baevski2024datacomp,li2024dolma}. Domain2Vec~\citep{sun2025domain2vec} further introduces dataset vectorization and distribution alignment, enabling mixture optimization without proxy models.

Online approaches adjust mixtures adaptively during training. Representative methods such as Group-DRO~\citep{sagawa2020dro} dynamically reweight domains to improve worst-case generalization under distribution shift. While effective, they rely on explicit domain labels and are costly to scale.

\subsection{Experimental Details}
\label{app:exp_details}
\paragraph{Benchmarks.}
We evaluate our method on nine diverse downstream benchmarks to assess its real-world impact. Guided by prior work \cite{mehta2024} and our own observations, we selected these tasks for their performance stability, excluding volatile benchmarks like RTE. The chosen tasks are HellaSwag \cite{zellers2019hellaswag}, PiQA \cite{bisk2020piqa}, OpenBookQA \cite{mihaylov2018openbookqa}, Lambada \cite{paperno2016lambada}, SciQ \cite{welbl2017sciq}, ARC-Easy \cite{clark2018arc}, COPA \cite{sarlin2020copa}, LogiQA \cite{liu2020logiqa}, and WinoGrande \cite{sakaguchi2021winogrande}. All evaluations use the \texttt{lm-eval-harness} \cite{gao2023lm-eval-harness}, and we report normalized accuracy where available, otherwise standard accuracy.

\paragraph{Baselines.}
To rigorously assess the effectiveness of our proposed method, DoGraph, we benchmark it against a diverse set of reweighting baselines spanning three levels of granularity. We first include the uniform mixing baseline, where all samples contribute equally, as a fundamental reference. We then compare DoGraph with state-of-the-art domain-level reweighting methods, including DoGE \cite{fan2023doge}, DoReMi \cite{xie2023doremi}, Regmix \cite{liu2024regmix}, and Data Mixing Law \cite{ye2024mixinglaws}. Finally, to evaluate performance at a finer granularity, we incorporate a representative sample-level reweighting approach, Dynamic Loss-based Sample Reweighting \cite{sow2025dynamic}.

\begin{table*}[!t]
\centering
\small
\resizebox{2\columnwidth}{!}{
\setlength{\tabcolsep}{4pt}
\begin{tabular}{l cccccc | cc | c | c}
\toprule
& \multicolumn{6}{c|}{\textbf{Commonsense / Reasoning}} & \multicolumn{2}{c|}{\textbf{RC}} & \multicolumn{1}{c|}{\textbf{LM}} & \textbf{Avg} \\
\cmidrule(lr){2-7} \cmidrule(lr){8-9} \cmidrule(lr){10-10} \cmidrule(lr){11-11}
\textbf{Method} & HellaSwag & PiQA & OBQA & COPA & LogiQA & WinoG & SciQ & ARC-E & Lambada &  \\
\midrule
\multicolumn{11}{c}{The Pile} \\
\midrule
Uniform              & 29.5 & 58.8 & 27.3 & 65.8 & 23.9 & 50.5 & 60.3 & 40.0 & 11.7 & 40.9 \\
Dynamic Loss-Based   & 29.0 & 57.7 & 26.4 & 64.3 & 22.8 & 49.3 & 60.0 & 38.9 & 10.2 & 39.9 \\
DoReMi               & 29.4 & 58.3 & 27.3 & \textbf{67.5} & 26.4 & 52.2 & 61.6 & 40.6 & 12.1 & 41.7 \\
DOGE                 & 29.2 & 58.5 & 27.1 & 64.5 & 23.2 & 49.8 & 60.1 & 40.0 & 11.7 & 40.5 \\
RegMix               & 29.2 & \textbf{59.3} & 27.3 & 65.2 & 25.8 & \textbf{53.1} & 62.8 & \textbf{41.7} & 14.2 & 42.1 \\
Data Mixing Law      & 29.2 & 58.8 & 26.9 & 67.2 & 23.6 & 50.4 & 58.6 & 39.0 & 11.9 & 40.6 \\
\midrule
\rowcolor{gray!20}
DoGraph (Ours)       & \textbf{29.8} & 59.2 & \textbf{27.8} & 65.0 & \textbf{28.3} & 51.2 & \textbf{66.1} & 39.2 & \textbf{15.9} & \textbf{42.5} \\
\bottomrule
\end{tabular}
}
\caption{
Downstream benchmark results (\textbf{accuracy \%}) on The Pile (LLaMA-1.1B).
Tasks are grouped into \textbf{Commonsense/Reasoning}, \textbf{Reading Comprehension}, and \textbf{Language Modeling}.
Our method, DoGraph, achieves consistently better and more balanced results across domains,
demonstrating its competitiveness and generalization ability. The best results are highlighted in bold.
}
\label{tab:benchmark_joint_1}
\end{table*}

\begin{table*}[!t]
\centering
\small
\resizebox{2\columnwidth}{!}{
\setlength{\tabcolsep}{4pt}
\begin{tabular}{l cccccc | cc | c | c}
\toprule
& \multicolumn{6}{c|}{\textbf{Commonsense / Reasoning}} & \multicolumn{2}{c|}{\textbf{RC}} & \multicolumn{1}{c|}{\textbf{LM}} & \textbf{Avg} \\
\cmidrule(lr){2-7} \cmidrule(lr){8-9} \cmidrule(lr){10-10} \cmidrule(lr){11-11}
\textbf{Method} & HellaSwag & PiQA & OBQA & COPA & LogiQA & WinoG & SciQ & ARC-E & Lambada &  \\
\midrule
\multicolumn{11}{c}{The Pile} \\
\midrule
Uniform              & 29.6 & 58.8 & 29.4 & 66.0 & 25.9 & 51.1 & 61.0 & 39.1 & 12.6 & 41.5 \\
Dynamic Loss-Based   & 29.3 & 58.1 & 29.6 & 66.1 & 25.0 & 52.5 
& 62.7 & 39.9 & 12.2 & 41.7 \\
DoReMi               & 29.6 & 58.4 & 29.8 & 66.0 & 24.9 & 51.4 & 61.1 & 40.5 & 12.8 & 
41.6 \\
DOGE                 & \textbf{29.7} & 56.9 & 29.2 & 64.0 & 25.5 & 50.6 & 61.7 & 40.6 & 11.9 & 41.1 \\
RegMix               & 29.4 & 59.5 & 29.4 & 66.5 & 25.1 & \textbf{53.6} & 62.5 & \textbf{41.2} & 12.3 & 42.2 \\
Data Mixing Law      & 29.3 & 58.4 & \textbf{30.2} & 65.9 & 25.6 & 51.3 & 61.3 & 40.2 & 12.1 & 41.6 
\\
\midrule
\rowcolor{gray!20}
DoGraph (Ours)       & 29.6 & \textbf{60.5} & 29.0 & \textbf{67.0} & \textbf{29.7} & 51.4 & \textbf{65.2} & 40.2 & \textbf{15.1} & \textbf{43.1} \\
\bottomrule
\end{tabular}
}
\caption{
Downstream benchmark results (\textbf{accuracy \%}) on The Pile (LLaMA-3.2-3B).
Tasks are grouped into \textbf{Commonsense/Reasoning}, \textbf{Reading Comprehension}, and \textbf{Language Modeling}.
Our method, DoGraph, achieves consistently better and more balanced results across domains,
demonstrating its competitiveness and generalization ability. The best results are highlighted in bold.
}
\label{tab:benchmark_joint_2}
\end{table*}

\begin{table*}[!t]
\centering
\small
\resizebox{2\columnwidth}{!}{
\setlength{\tabcolsep}{4pt}
\begin{tabular}{l cccccc | cc | c | c}
\toprule
& \multicolumn{6}{c|}{\textbf{Commonsense / Reasoning}} & \multicolumn{2}{c|}{\textbf{RC}} & \multicolumn{1}{c|}{\textbf{LM}} & \textbf{Avg} \\
\cmidrule(lr){2-7} \cmidrule(lr){8-9} \cmidrule(lr){10-10} \cmidrule(lr){11-11}
\textbf{Method} & HellaSwag & PiQA & OBQA & COPA & LogiQA & WinoG & SciQ & ARC-E & Lambada &  \\
\midrule
\multicolumn{11}{c}{SlimPajama} \\
\midrule
Uniform              & 26.0 & 55.4 & 13.8 & 57.2 & 22.8 & 49.3 & 32.6 & 30.6 & 12.0 & 33.3 \\
Dynamic Loss-Based   & 26.2 & 56.1 & 13.2 & 55.3 & 26.0 & 49.2 & \textbf{53.8} & 31.8 & 12.6 & 36.2 \\
DoReMi               & 26.1 & 55.7 & 12.3 & 53.5 & \textbf{26.8} & 48.8 & 52.4 & 30.9 & 12.4 & 35.4 \\
DOGE                 & 26.2 & 55.0 & 14.4 & \textbf{60.5} & 23.5 & 49.0 & 31.1 & 30.8 & 11.4 & 33.5 \\
RegMix               & 26.0 & 54.3 & 13.3 & 58.0 & 24.1 & \textbf{49.8} & 38.7 & 29.8 & 12.5 & 34.1 \\
Data Mixing Law      & 26.1 & 56.3 & 13.4 & 59.2 & 24.5 & 48.9 & 39.6 & 30.1 & 12.4 & 34.5 \\
\midrule
\rowcolor{gray!20}
DoGraph (Ours)       & \textbf{26.3} & \textbf{57.5} & \textbf{14.6} & 58.0 & 26.0 & 49.7 & \textbf{53.8} & \textbf{32.3} & \textbf{12.8} & \textbf{36.4} \\
\bottomrule
\end{tabular}
}
\caption{
Downstream benchmark results (\textbf{accuracy \%}) on SlimPajama (GPT-2 Small).
Tasks are grouped into \textbf{Commonsense/Reasoning}, \textbf{Reading Comprehension}, and \textbf{Language Modeling}.
Our method, DoGraph, achieves consistently better and more balanced results across domains,
demonstrating its competitiveness and generalization ability. The best results are highlighted in bold.
}
\label{tab:benchmark_joint_3}
\end{table*}

\paragraph{Training Datasets.}
Our training data strategy is designed to align dataset scale with model capacity.  
For all GPT-2 models, we utilize the \textbf{SlimPajama-6B} dataset \cite{yoon2023slimpajama}, a 6-billion-token corpus comprising seven diverse domains: ArXiv, Books, Common Crawl, C4, GitHub, StackExchange, and Wikipedia.  
The byte proportion of each source is detailed in Table~\ref{tab:slimpajama_6b_composition}, illustrating the composition of the data mixture used for training.  
For all LLaMA models, we  conduct our experiments using the domains of the Pile dataset \cite{gao2021pile} depicted in Table~\ref{tab:the_pile_composition}. Due to copyright concerns, we utilize the 17 subsets available on HuggingFace that do not violate copyright issues.
These datasets provide a balanced and diverse text distribution suitable for evaluating cross-domain generalization in medium-scale language models.

\begin{table}[h!]
    \centering
    \begin{tabular}{lr}
        \toprule
        \textbf{Data Source} & \textbf{Byte Proportion} \\
        \midrule
        Common Crawl    & 54.1\% \\
        C4              & 28.7\% \\
        GitHub          & 4.2\%  \\
        Books           & 3.7\%  \\
        ArXiv           & 3.4\%  \\
        Wikipedia       & 3.1\%  \\
        StackExchange   & 2.8\%  \\
        \bottomrule
    \end{tabular}
\caption{Byte proportion of data sources in the SlimPajama-6B dataset.}
\label{tab:slimpajama_6b_composition}
\end{table}

\definecolor{rowgray}{gray}{0.9}
\newcommand{\gc}{\cellcolor{rowgray}}

\begin{table}[ht]
\centering
\begin{tabular}{lr}
\toprule
\textbf{Component} & \textbf{Effective Size} \\
\midrule
Pile-CC & 227.12 GiB \\
PubMed Central & 180.55 GiB \\
\rowcolor{rowgray} Books3 & 151.44 GiB \\
\rowcolor{rowgray} OpenWebText2 & 125.54 GiB \\
ArXiv & 112.42 GiB \\
Github & 95.16 GiB \\
FreeLaw & 76.73 GiB \\
Stack Exchange & 64.39 GiB \\
USPTO Backgrounds & 45.81 GiB \\
PubMed Abstracts & 38.53 GiB \\
Gutenberg (PG-19) & 27.19 GiB \\
\rowcolor{rowgray} OpenSubtitles & 19.47 GiB \\
Wikipedia (en) & 19.13 GiB \\
DM Mathematics & 15.49 GiB \\
Ubuntu IRC & 11.03 GiB \\
\rowcolor{rowgray} BookCorpus2 & 9.45 GiB \\
EuroParl & 9.17 GiB \\
HackerNews & 7.80 GiB \\
\rowcolor{rowgray} YoutubeSubtitles & 7.47 GiB \\
PhilPapers & 4.76 GiB \\
NIH ExPorter & 3.79 GiB \\
Enron Emails & 1.76 GiB \\
\bottomrule
\end{tabular}
\caption{Overview of the Pile dataset with datasets that are no longer available due to copyright issues marked in gray. Merged into a single column list.}
\label{tab:the_pile_composition}
\end{table}


\paragraph{Model Architecture.}
Following prior studies~\cite{liu2024regmix, sow2025dynamic}, we consider both model architecture and model scale in our evaluation, as summarized in Table~\ref{tab:model-arch}.  
Specifically, we evaluate two decoder-only Transformer models based on GPT-2 architecture and two models based on LLaMA architecture, ranging from lightweight to medium scales.

\begin{table}[htbp]
\centering
\resizebox{\columnwidth}{!}{%
\begin{tabular}{lcccc}
\toprule
 & GPT-2 Small & GPT-2 Medium & LLaMA-1.1B & LLaMA-3.2-3B\\
\midrule
Parameters      & 210M  & 300M  & 1.1B  & 3B  \\
Layers          & 24    & 36  & 22  & 28    \\
Attention Heads & 16    & 24  & 32  & 24    \\
Embedding Dim.  & 768   & 768  & 2048  & 8192   \\
Hidden Dim.     & 3072  & 3072  & 2048  & 3072  \\
Max Seq. Length & 512   & 512  & 2048  & 131072 \\
\bottomrule
\end{tabular}%
}
\caption{Model architectures used in our experiments.}
\label{tab:model-arch}
\end{table}

\paragraph{Training Process.}
Following standardized practices in prior work, we train all models under protocols summarized in Table~\ref{tab:train-hparams-vertical}.  
Specifically, we adopt a linear warmup cosine schedule with identical weight decay (0.01) and gradient clipping (1.0) across all model scales,  
while adjusting batch size and training steps according to model capacity.  
This setup ensures that each model is trained sufficiently to convergence.

\begin{table}[htbp]
\centering
\resizebox{\columnwidth}{!}{%
\begin{tabular}{lcccc}
\toprule
 & GPT-2 Small & GPT-2 Medium & TinyLLaMA-1.1B & LLaMA-3.2-3B\\
\midrule
Minibatch Size            & 48   & 48   & 64   & 64 \\
Learning Rate $(\times 10^{-3})$ & 0.50 & 0.50 & 0.50 & 0.50 \\
Learning Rate End $(\times 10^{-4})$ & 1.0 & 1.0 & 1.0 & 1.0 \\
Warmup Steps              & 500  & 500  & 500  & 500  \\
$r$                       & 0.4  & 0.4  & 0.4  & 0.4  \\
Training Steps            & 20{,}000 & 20{,}000 & 25{,}000 & 25{,}000 \\
Total Documents Seen       & 960{,}000 & 960{,}000 & 1280{,}000 & 1280{,}000 \\
\bottomrule
\end{tabular}%
}
\caption{Training hyperparameters for GPT-2 and LLaMA models in our benchmark evaluations.}
\label{tab:train-hparams-vertical}
\end{table}

\subsection{DoGraph Pipeline}
\label{sec:dograph_pipeline}
Formalized in Algorithm~\ref{alg:dograph}, the process begins by projecting high-dimensional per-sample gradients $g_i^{(t)} \in \mathbb{R}^d$ into a lower-dimensional subspace $\tilde g_i^{(t)} \in \mathbb{R}^k$ using a random Gaussian matrix $R^{(t)}$, where we set $k=5000$ for both SlimPajama and The Pile to preserve the gradient manifold's geometric properties per the Johnson-Lindenstrauss Lemma. Subsequently, we identify latent optimization structures by applying K-means clustering to these projected signals, partitioning the mini-batch into $m=11$ model-centric domains $\{D_j^{(t)}\}$ and computing their respective centroid gradients $\bar g_j^{(t)}$. Finally, importance weights $w^{(t)} \in \Delta^{m-1}$ are determined by solving the auxiliary objective $\mathcal{L}_{\text{opt}}$, and the model parameters $\theta$ are updated via the weighted aggregate $\sum_{j=1}^m w_j^{(t)} \bar g_j^{(t)}$, effectively decoupling training dynamics from static, pre-defined domain labels.
\begin{algorithm}[htbp]
\caption{Dograph Pipeline}
\label{alg:dograph}
\KwIn{Training data $\mathcal D$, parameters $\theta$, number of clusters $m$, projection dimension $k$, epochs $T$, learning rate $\eta$}
\KwOut{Trained parameters $\theta^*$, domain weights $\{w^{(t)}\}$}

\For{$t=1$ \KwTo $T$}{
  Sample random projection matrix $R^{(t)} \in \mathbb{R}^{d \times k}$ with $R_{pq}^{(t)} \sim \mathcal{N}(0, 1/k)$\;
  Compute per-sample gradients $g_i^{(t)} = \nabla_\theta L(x_i, y_i; \theta^{(t-1)})$\;
  Project gradients: $\tilde g_i^{(t)} = R^{(t)\top} g_i^{(t)}$\;
  Cluster $\{\tilde g_i^{(t)}\}$ into $m$ domains $\{D_1^{(t)}, \dots, D_m^{(t)}\}$\;
  Compute domain mean gradients $\bar g_j^{(t)} = \frac{1}{|D_j^{(t)}|} \sum_{i \in D_j^{(t)}} \tilde g_i^{(t)}$\;
  Optimize weights $w^{(t)} = \arg\min_{w \in \Delta^{m-1}} \mathcal{L}_{\text{opt}}\!\left(\sum_{j=1}^m w_j \bar g_j^{(t)}\right)$\;
  Update model parameters: $\theta^{(t)} \gets \theta^{(t-1)} - \eta \sum_{j=1}^m w_j^{(t)} \bar g_j^{(t)}$\;
}
\Return{$\theta^{(T)}, \{w^{(t)}\}_{t=1}^T$}
\end{algorithm}

\subsection{Scaling to Larger Datasets and Model Sizes}
\label{sec:Scaling to Larger Datasets and Model Sizes}
We report results on GPT-2 models from 210M to 300M parameters and a 6B-token SlimPajama subset, as shown in Table~\ref{tab:benchmark_joint_3}. DoGraph is scale-free and does not rely on any model-size–specific assumptions. All components, including gradient extraction, random projection, clustering, and domain-level optimization, operate directly on per-step gradients and thus scale linearly with model size. The method does not require proxy models, validation-model fitting, or domain-specific metadata, making it naturally compatible with billion-parameter LLMs.
To further prove these, we pretrain LLaMA-1.1B and LLaMA-3B from scratch under the same DoGraph pipeline, as shown in Table~\ref{tab:benchmark_joint_1} and Table~\ref{tab:benchmark_joint_2}.

\subsection{Clustering Visualization}
\label{sec:Clustering Visualization}
As shown in Figure~\ref{fig:model_centric_domains}, while human-defined domains (indicated by colors) become indistinguishable later in training, DoGraph successfully extracts $m$ latent structures from this mixture, proving that model-centric domains are composed of heterogeneous data sources.

\begin{figure*}[t]
\centering
\includegraphics[width=\textwidth]{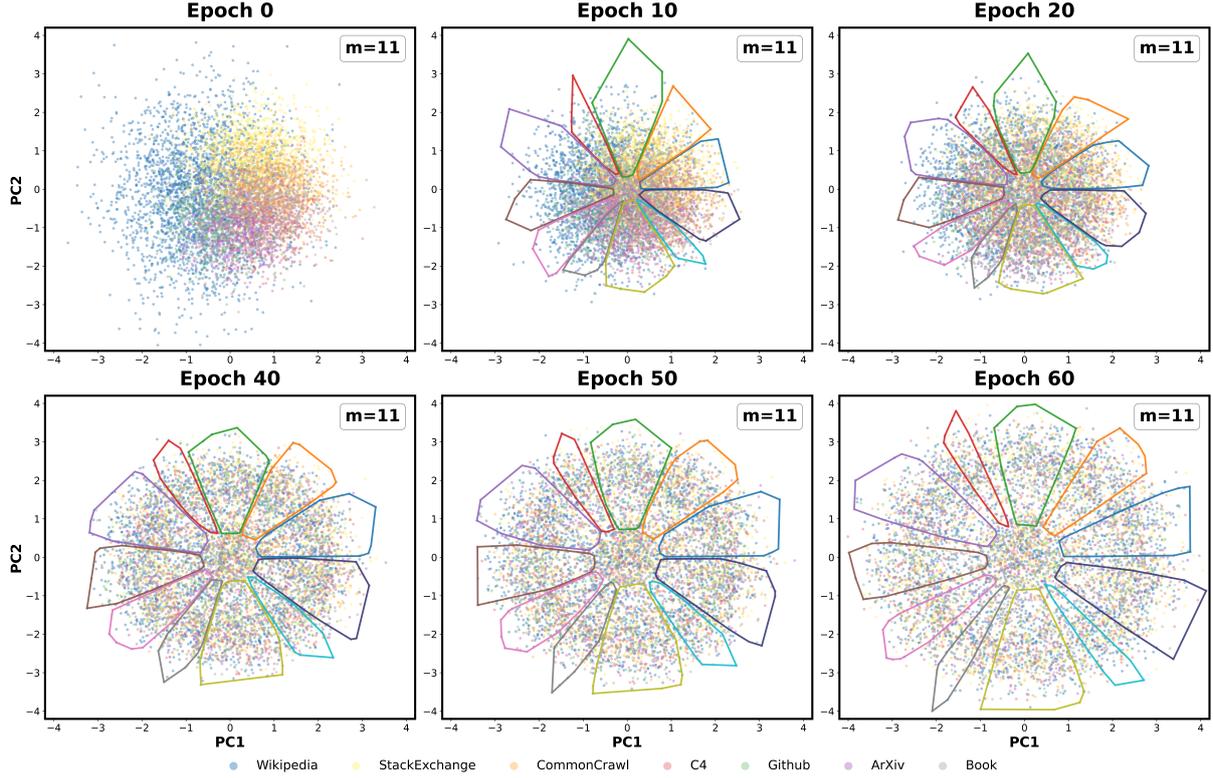}
\caption{
Evolution of per-sample gradients and the emergence of model-centric structures. Points are colored by their original source datasets. Initially, gradients are separated by domain bias. Over time, the model homogenizes its perception of these sources, leading to significant overlap. Despite this mixing, DoGraph identifies m=11 distinct model-centric domains within the gradient space. Experiments use 20\% of SlimPajama trained on GPT2-Mini.
}
\label{fig:model_centric_domains}
\end{figure*}

\subsection{More Analysis about the Choice of Optimization Function}
\label{sec:more_analysis_about}
At each training epoch, the DoGraph framework computes domain mean gradients $\{\bar g_j\}_{j=1}^m$
and determines their adaptive weights $w \in \Delta^{m-1}$ by minimizing an auxiliary objective $\mathcal{L}_{\text{opt}}$.
Since $m$ (the number of domains) is typically small, this optimization occurs in a low-dimensional space and can be efficiently solved in closed or iterative form.
We discuss several representative objectives and their corresponding solvers below.

\noindent\textbf{Gradient variance minimization.}
To balance the learning progress across domains, one may minimize the variance of gradient magnitudes while maintaining the global descent direction:
\[
\mathcal{L}_{\text{opt}}(w)
= \mathrm{Var}_{j}\!\big[\|\bar g_j\|_2\big]
+ \lambda \Big\|\sum_{j=1}^m w_j \bar g_j\Big\|_2^2.
\]
This convex quadratic problem can be solved by projected gradient descent or quadratic programming with a simplex constraint.

\noindent\textbf{Robust min--max objective.}
When robustness against hard or under-represented domains is desired, one may adopt a distributionally robust formulation:
\[
\mathcal{L}_{\text{opt}}(w)
= \tau \log \!\sum_{j=1}^m 
   \exp\!\big(\|\bar g_j\|_2 / \tau\big),
\]
which smoothly approximates $\max_j \|\bar g_j\|_2$.
The optimal weights admit a closed-form softmax solution
$w_j \propto \exp(\|\bar g_j\|_2 / \tau)$.

\noindent\textbf{Gradient alignment regularization.}
To encourage consistent update directions across domains,
we define
\[
\mathcal{L}_{\text{opt}}(w)
= - \sum_{j=1}^m w_j 
    \cos(\bar g_j,\;\bar g),
\quad 
\bar g = \sum_{j=1}^m w_j \bar g_j.
\]
Although non-convex due to the dependence of $\bar g$ on $w$,
it can be efficiently solved by a few fixed-point iterations:
each step updates $w_j$ in proportion to the cosine similarity between $\bar g_j$ and the current aggregate $\bar g$.

\noindent\textbf{Domain uncertainty weighting.}
Alternatively, if each domain exhibits distinct gradient variability, 
we estimate its intra-domain variance
$\sigma_j^2 = \mathrm{Var}_{i\in D_j}\!\big[\|g_i - \bar g_j\|_2^2\big]$
and assign weights inversely proportional to it:
\[
\mathcal{L}_{\text{opt}}(w)
= \Big\|\sum_j w_j \bar g_j\Big\|_2^2
+ \beta \sum_j \sigma_j^2 w_j.
\]
This convex quadratic form admits a closed-form Newton update or
can be solved by a projected Frank--Wolfe method.

Table~\ref{tab:opt_objectives} summarizes the computational complexity of each optimization objective
and the corresponding validation perplexity PPL.
All variants share the same backbone and differ only in the choice of $\mathcal{L}_{\text{opt}}$.
The robust softmax objective achieves the lowest computational cost,
while the uncertainty-weighted variant attains the best overall performance.

\begin{table}[h]
\centering
\caption{Comparison of optimization objectives in DoGraph.}
\label{tab:opt_objectives}
\small
\begin{tabular}{lcc}
\toprule
\textbf{Method} & \textbf{Complexity} & \textbf{PPL} \\
\midrule
DoGraph (variance)         & $O(m^2)$ & 3.24 \\
DoGraph (robust softmax)   & $O(m)$   & 3.31 \\
DoGraph (alignment)        & $O(m^2)$ & 3.15 \\
DoGraph (uncertainty)      & $O(m^2)$ & \textbf{3.09} \\
\bottomrule
\end{tabular}
\end{table}

\noindent
Among all variants, \textbf{DoGraph (uncertainty)} achieves the lowest perplexity,
indicating that weighting domains by intra-domain gradient stability
provides the most consistent optimization dynamics.

\subsection{Impact of Cluster Granularity $m$.}
\label{sec:impact}
We investigate the sensitivity of model performance to the number of clusters m. As illustrated in Figure 1, the validation perplexity exhibits a clear U-shaped trend with respect to m. Performance initially improves as m increases from 7 to 11, suggesting that moderately finer-grained, model-centric domains better capture coherent gradient structures and facilitate optimization. However, further increasing m beyond 11 leads to a significant performance degradation. This decline is likely due to over-partitioning, which fragments the gradient space and splits coherent patterns into inconsistent components, thereby weakening signal consistency. Consequently, we select m=11 as our default setting for all subsequent experiments.
\begin{figure}[t] 
    \centering
    \includegraphics[width=1\linewidth]{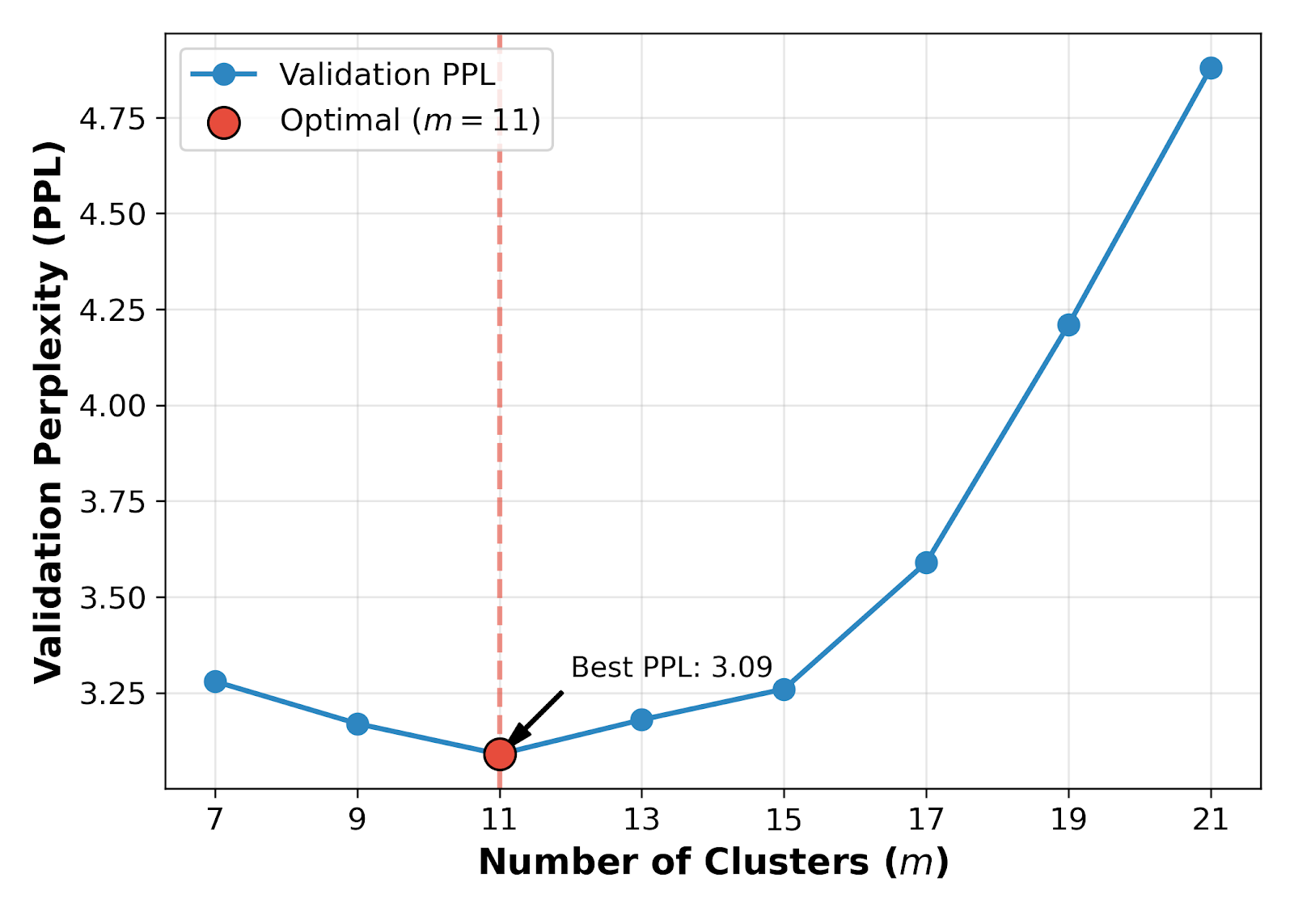}
    \caption{Impact of cluster granularity $m$ on validation perplexity. The U-shaped curve demonstrates that $m=11$ provides the optimal balance; insufficient granularity fails to resolve gradient structures, while excessive partitioning leads to signal inconsistency.}
    \label{fig:cluster_granularity}
\end{figure}
\subsection{Computational Efficiency Analysis}
As shown in Figure~\ref{fig:efficiency_h200}, \texttt{dograph} achieves state-of-the-art performance while introducing a modest and practical computational overhead. On a $2\times$ H200 GPU cluster, our method completes pre-training in 20.37 hours, corresponding to a 4.51\% increase in runtime compared to \texttt{regmix}. This incremental cost falls within the commonly accepted budget for large-scale pre-training.
\begin{figure}[t]
    \centering
    \includegraphics[width=0.9\linewidth]{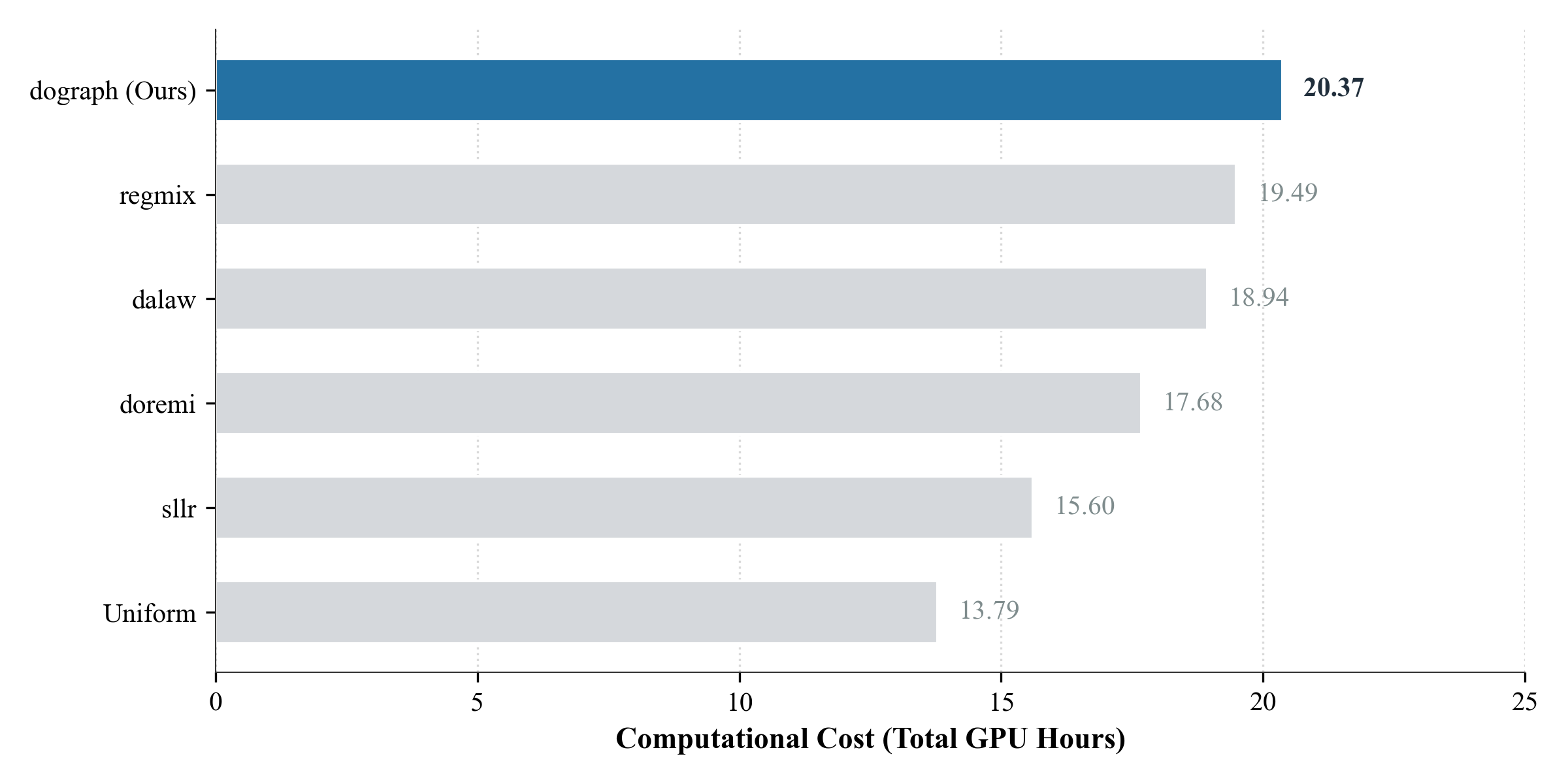}
    \caption{Pre-training GPT-2 Mini on SlimPajama under a 100B-Token Computational Budget. We report the total training time (GPU hours) using $2\times$ NVIDIA H200 GPUs. While our \texttt{dograph} method introduces a sophisticated data-driven decision process, the resulting overhead is minimal (only 4.51\% over \texttt{regmix}), while establishing a new  SOTA performance baseline. The marginal increase in budget is well-justified by the superior convergence quality and data selection efficiency.}
    \label{fig:efficiency_h200}
\end{figure}

\subsection{Proofs}
\begin{tcolorbox}[colback=lightgray, colframe=lightgray, sharp corners=all, boxrule=0mm, boxsep=0.5mm, left=1.5mm, right=1.5mm, top=1.5mm, bottom=1.5mm]
\begin{assumption}[Linearized Attention Mechanism]\label{assump:attn}
Let \(X \in \mathbb{R}^{n \times d}\) denote the sequence representation.
We define the projected queries, keys, and values as
\(Q = XW_Q\), \(K = XW_K\), and \(V = XW_V\),
and the scaled similarity matrix as
\(S = \tfrac{1}{\sqrt{d_k}} QK^\top\).
The row-wise softmax of \(S\) is approximated by its first-order linearization:
\(
P ~\approx~ A + \tfrac{1}{\tau n} C S,
\quad \text{where} \quad
A = \tfrac{1}{n}\mathbf{1}\mathbf{1}^\top, \ 
C = I - \tfrac{1}{n}\mathbf{1}\mathbf{1}^\top.
\)
Consequently, the attention output satisfies
\(O = PV \approx AV + TQK^\top V\),
where \(T = \tfrac{1}{\tau n \sqrt{d_k}}\,C\).
\end{assumption}
\end{tcolorbox}

\begin{tcolorbox}[colback=lightgray, colframe=lightgray, sharp corners=all, boxrule=0mm, boxsep=0.5mm, left=1.5mm, right=1.5mm, top=1.5mm, bottom=1.5mm]
\begin{assumption}[Linear Output Transformations]\label{assump:out}
The attention output \(O\) is passed through two linear mappings:
\(H = OW_O\) and \(Z = HW\).
The model prediction is obtained via
\(\Pi = \mathrm{softmax}(Z)\),
which represents the token-level probability distribution.
\end{assumption}
\end{tcolorbox}

\begin{tcolorbox}[colback=lightgray, colframe=lightgray, sharp corners=all, boxrule=0mm, boxsep=0.5mm, left=1.5mm, right=1.5mm, top=1.5mm, bottom=1.5mm]
\begin{assumption}[Upstream Gradients and Mismatch Tensor]\label{assump:grad}
Given the ground-truth label matrix \(Y\),
the upstream gradients are defined as
\(G_Z = \Pi - Y\),
\(G_H = G_Z W^\top\),
and \(G_O = G_H W_O^\top\).
We further define the mismatch tensor
\(R = (\Pi - Y)M\),
where \(M = W^\top W_O^\top\).
\end{assumption}
\end{tcolorbox}

\begin{tcolorbox}[colback=lightgray, colframe=lightgray, sharp corners=all, boxrule=0mm, boxsep=0.5mm, left=1.5mm, right=1.5mm, top=1.5mm, bottom=1.5mm]
\begin{assumption}[Regularity Conditions]\label{assump:reg}
All expectations involved in subsequent derivations are assumed to exist,
and the per-sample gradients are square-integrable.
\end{assumption}
\end{tcolorbox}

\begin{proof}[Per-sample gradients and proof of Theorem~3.2 ]
With $dL=\langle G_O,dO\rangle$ and $O\approx AV+TQK^\top V$, the per-sample gradients are
\[
\left\{
\begin{aligned}
\frac{\partial L}{\partial V} &= A^\top G_O + KQ^\top T^\top G_O,\\[3pt]
\frac{\partial L}{\partial Q} &= (T^\top G_O)\,V^\top K,\\[3pt]
\frac{\partial L}{\partial K} &= V\,G_O^\top T\,Q,\\[3pt]
\frac{\partial L}{\partial W} &= H^\top G_Z
\end{aligned}
\right.
\]

\vspace{4pt}

\[
\left\{
\begin{aligned}
\frac{\partial L}{\partial W_V} &= X^\top\!\big(A^\top G_O + KQ^\top T^\top G_O\big),\\[3pt]
\frac{\partial L}{\partial W_Q} &= X^\top\!\big((T^\top G_O)\,V^\top K\big),\\[3pt]
\frac{\partial L}{\partial W_K} &= X^\top\!\big(V\,G_O^\top T\,Q\big),\\[3pt]
\frac{\partial L}{\partial W_O} &= O^\top G_H
\end{aligned}
\right.
\]


From Assumption~\ref{assump:grad}, the upstream gradient can be written as
\(G_O = (\Pi - Y)W^\top W_O^\top = R\).
Substituting this into the above expressions shows that
all per-sample gradients are \emph{linear} functions of \(R\):
\[
\nabla_{W_b}L(x,y;\theta)
~=~ \mathsf{Lin}_b(X,Q,K,V,T)\,[\,R(x,y)\,],
\]
where \(\mathsf{Lin}_b(\cdot)\) denotes a matrix-valued linear operator determined
only by the forward pass variables \(X,Q,K,V,T\). Using the identity \(\mathrm{vec}(UGV)=(V^\top\!\otimes U)\mathrm{vec}(G)\),
each matrix gradient can be rewritten in vectorized form as
\[
\begin{aligned}
g_b(s) &:= \mathrm{vec}\!\big(\nabla_{W_b}L(x,y;\theta)\big)
~=~ \mathcal{L}_b(x)\,\rho(s),\\[3pt]
\rho(s) &:= \mathrm{vec}\!\big(R(s)\big).
\end{aligned}
\]
Here, $\mathcal{L}_b(x)$ absorbs all Kronecker factors
(e.g., $X^\top$, $T^\top$, $K$, $Q$) from the explicit gradient expressions.
Hence, for every parameter block \(b\),
the sample-wise gradient is a linear transformation of the mismatch vector~$\rho(s)$.

Define the expected gradient under a data distribution \(P\) as
\(\bar g_b(P):=\mathbb{E}_{s\sim P}[g_b(s)]\).
By linearity of expectation (Bochner integral in finite dimensions),
\[
\bar g_b(P_1)-\bar g_b(P_2)
~=~\int g_b(s)\,(P_1-P_2)(ds).
\]

The inner product between two per-sample gradients naturally defines
\[
\begin{aligned}
k_b(s,s') &:= \langle g_b(s),\, g_b(s') \rangle \\[3pt]
&= \rho(s)^\top\, \mathcal{L}_b(x)^\top \mathcal{L}_b(x')\, \rho(s').
\end{aligned}
\]
Because $k_b$ is an inner product in feature space, it is positive semidefinite. Applying Fubini--Tonelli and the bilinearity of the inner product yields
\[
\begin{aligned}
&\big\|\bar g_b(P_1)-\bar g_b(P_2)\big\|_2^2\\
&=\Big\langle\int g_b\,d(P_1\!-\!P_2),\ \int g_b\,d(P_1\!-\!P_2)\Big\rangle\\
&=\iint \langle g_b(s),g_b(s')\rangle\,(P_1\!-\!P_2)(ds)\,(P_1\!-\!P_2)(ds')\\
&=\mathbb{E}_{P_1,P_1}[k_b]+\mathbb{E}_{P_2,P_2}[k_b]-2\,\mathbb{E}_{P_1,P_2}[k_b],
\end{aligned}
\]
which is exactly $\mathrm{MMD}^2_{k_b}(P_1,P_2)$ by definition.
Finally, the positive semidefiniteness of $k_b$ follows from
\[
\sum_{i,j}\alpha_i\alpha_j k_b(s_i,s_j)
=\Big\|\sum_i \alpha_i g_b(s_i)\Big\|_2^2\ge 0.
\]
This completes the proof.
\end{proof}

\end{document}